\documentclass[10pt]{article}

\usepackage{amsmath,amsfonts,amsthm,amssymb}
\usepackage[margin=1.25in]{geometry}

\usepackage{graphics, graphicx, xcolor}
\usepackage{hyperref}
\usepackage{url}
\usepackage{latexsym}
\usepackage{array}
\usepackage{verbatim}
\usepackage[shortlabels,]{enumitem}
\usepackage[algo2e]{algorithm2e}
\usepackage{multirow}

\usepackage{booktabs}
\usepackage{cellspace}
\usepackage{cleveref}
\usepackage{csquotes}

\newtheorem{thm}{Theorem}
\newtheorem{lem}[thm]{Lemma}
\newtheorem{prop}[thm]{Proposition}

\newtheorem{observation}[thm]{Observation}
\newtheorem{defn}[thm]{Definition}

\newtheorem{assumption}{Assumption}

\newcommand{\pred}{\textsc{PRED}}

\newcommand{\clip}{\mbox{clip}}

\newcommand{\va}{\mathbf{a}}

\newcommand{\vF}{\mathbf{F}} 
  
\newcommand{\vh}{\mathbf{h}}
\newcommand{\vx}{\mathbf{x}}
\newcommand{\vb}{\mathbf{b}}

\newcommand{\vg}{\mathbf{g}} 

\newcommand{\vp}{\mathbf{p}}

\newcommand{\vS}{\mathbf{S}}

\newcommand{\vz}{\mathbf{z}}

\newcommand{\vzero}{\mathbf{0}}
\newcommand{\vone}{\mathbf{1}}

\DeclareMathOperator*{\argmin}{arg\,min}

\DeclareMathOperator{\sgn}{sgn}

\newcommand{\RR}{\mathbb{R}}      

\newcommand{\vnorm}[1]{\left\lVert#1\right\rVert} 
\newcommand{\abs}[1]{\left| #1 \right|}

\newcommand{\pderiv}[2]{\frac {\partial \left[ #1 \right]} {\partial #2}}

\newcommand{\cH}{\mathcal{H}}

\newcommand{\lrp}[1]{\left(#1\right)}
\newcommand{\lrb}[1]{\left[#1\right]}

\title{Learning to Abstain from Binary Prediction}
\usepackage{times}

\author{
  Akshay Balsubramani \\
  Computer Science and Engineering\\
  University of California, San Diego\\
  \url{abalsubr@ucsd.edu}
}

\date{}

\begin{document}

\maketitle

\begin{abstract}
A binary classifier capable of abstaining from making a label prediction 
has two goals in tension: 
minimizing errors, and avoiding abstaining unnecessarily often. 
In this work, we exactly characterize the best achievable tradeoff between these two goals in a general semi-supervised setting, 
given an ensemble of predictors of varying competence as well as unlabeled data on which we wish to predict or abstain. 
We give an algorithm for learning a classifier in this setting which trades off its errors with abstentions in a minimax optimal manner, 
is as efficient as linear learning and prediction, and is demonstrably practical. 
Our analysis extends to a large class of loss functions and other scenarios, 
including ensembles comprised of \enquote{specialists} that can themselves abstain.
\end{abstract}


\section{Introduction}

Consider a general practice physician treating a patient with unusual or ambiguous symptoms. 
The general practitioner often does not have the capability to confidently diagnose such an ailment. 
The doctor is faced with a difficult choice: either \emph{commit} to a potentially erroneous diagnosis and act on it, which can have catastrophic consequences; 
or \emph{abstain} from any such diagnosis and refer the patient to a specialist or hospital instead, which is safer but will certainly cost extra time and resources.

Such a situation motivates the study of classifiers which are able not only to form a hypothesis about the
correct classification, but also abstain entirely from making a prediction. 
A sufficiently self-aware abstaining classifier might abstain on examples on which it is most unsure about the label, 
lowering the average prediction error it suffers when it does commit to a prediction. 
Like the doctor in the example, however, there is typically no use in abstaining on all data, so the amount of overall abstaining is somehow restricted. 
The classifier must allocate limited abstentions where they will most reduce error. 

There has been much historical work in decision theory and machine learning 
on learning such abstaining classifiers (e.g. \cite{C57, C70, HLS97, T00}),  
where the setting is often dubbed \enquote{classification with a reject option.}
The central focus of this work is to characterize a \emph{tradeoff}, between the abstain rate 
and the probability of error when committing to predict (the error rate). 
The optimal tradeoff is achieved by the set of classifiers which minimize error rate for a given abstain rate. 
Such classifiers are Pareto optimal, and the set of them is called the \emph{Pareto frontier} between abstain and error rates.

The tradeoff has previously been examined theoretically by learning classifiers that are overly conservative by design in deciding to abstain, 
in order to prove error guarantees. 
So the Pareto frontier between abstaining and erring has only been determined approximately or at a handful of points, 
despite a spate of recent work (\cite{BW08, EYW10, ZC14}; Sec. \ref{sec:relwork}). 

A primary contribution of this work is to describe the Pareto frontier \emph{completely} for a very general semi-supervised learning scenario, 
in which i.i.d. data are available in both labeled and unlabeled form. 
We build an aggregated abstaining classifier from a given ensemble of predictors, 
with the labeled data used to estimate their predictive power, 
and with access to a large unlabeled test dataset on which the ensemble's predictions are known. 

This entirely plausible situation adds only the abundant unlabeled data to the prototypical fully supervised learning setting, 
in which the goal is to learn a non-abstaining classifier. 
In the supervised setting and given only ensemble error rates, 
it has long been known that the best safe strategy is to predict according to the best single classifier -- that is, empirical risk minimization (ERM; \cite{V82}). 
But adding unlabeled data makes learning much easier. 
Recent work (\cite{BF15}, henceforth referred to as BF) derives the minimax optimal non-abstaining algorithm 
for the previously described semi-supervised scenario, 
and proves that it always performs at least as well as ERM. 

This paper generalizes the BF work 
by exactly specifying the Pareto optimal frontier when learning a classifier that can abstain in our semi-supervised setting.\footnote{To fix terminology, 
we learn a \emph{classifier} capable of two types of output on an input data point: 
\emph{abstaining}, and \emph{predicting} a label (that is, not abstaining). 
This is to avoid confusion, since there appear to be no standard names for these concepts in the literature. }
We also give an efficient method for learning the abstaining classifiers realizing this optimal tradeoff, 
which achieve the lowest possible guarantee on error rate for any given abstain rate. 
The resulting error guarantees as a function of abstain rate, and their dependence on the structure of the unlabeled ensemble predictions, 
are without precedent in the literature and are all unimprovable in this setting, by virtue of the minimax arguments we use to derive them. 

This paper builds towards such results in the initial sections, first giving the semi-supervised setup inherited from BF 
in Section \ref{sec:setup}. 
This is followed by an introduction to our techniques on a common model of abstaining: 
as a third possible outcome with a specified cost independent of the true labels (Section \ref{sec:abstcost}). 
We derive the abstaining classifier with the best worst-case loss guarantee among \emph{all possible} classifiers in our semi-supervised setting. 
It can be learned with a simple convex optimization, as scalable as a linear learning algorithm like the perceptron. 

In Section \ref{sec:abstconstr}, we establish the optimal tradeoff between abstaining and erring. 
To be specific, the Pareto frontier is a one-dimensional curve, parametrized by the allowed abstain rate $\alpha$ 
(Fig. \ref{fig:absterrtradeoff}, derived in Thm. \ref{thm:maintradeoff}). 
Given any $\alpha$, we specify the classifier with the lowest possible error bound among all classifiers with abstain rate $\alpha$, 
thereby giving all points on the optimal tradeoff curve. 
Our analysis shows that this Pareto frontier is identical whether abstaining is penalized explicitly with a known cost (as in Sec. \ref{sec:abstcost}) 
or is simply restricted to not occur too often overall (Sec. \ref{sec:abstconstr}), 
the two prevailing models of abstaining in the literature. 
Building on previous sections, we reach an efficiently learnable characterization of the Pareto optimal classifiers on the frontier. 
Proofs of all results in the paper are deferred to Appendix \ref{sec:apdxproofs}. 


\begin{figure}[t]
\centering
\label{fig:absterrtradeoff}
\includegraphics[height=1.9in, width=0.45\linewidth]{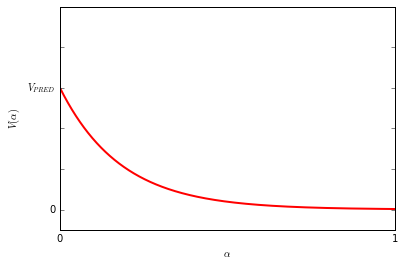}
\caption{Schematic of the abstain/error Pareto frontier in our setting (formal statement in Thm. \ref{thm:maintradeoff}): 
minimum achievable error $V (\alpha)$ for an aggregated classifier with allowed abstain rate $\leq \alpha$. }
%
\end{figure}

Having established the tradeoff and how to realize it optimally and efficiently, 
we discuss further extensions and related work in Sec. \ref{sec:discrelwork}. 
We then implement this paper's algorithms and empirically evaluate them against common abstaining baselines as a proof of concept (Sec. \ref{sec:experiments}), 
to illustrate some of the new algorithms' practical advantages in making scalable abstaining decisions with unlabeled data.

To show the versatility of the techniques in this paper, we also generalize them significantly
to a very large class of other loss functions 
(Appendix \ref{sec:general} for space reasons). 
This includes all convex surrogate losses used for ERM, as well as a multitude of non-convex losses. 
All the contributions of this paper, including the algorithms and characterization of the optimal tradeoff, extend to these losses as well.


\section{A Semi-Supervised Setting}
\label{sec:setup}

A natural starting point for our development is the non-abstaining scenario typically studied, 
which is a special case of our problem when the learned classifier is prevented from abstaining at all. 
So in this section, we briefly recapitulate the setup of BF 
for the non-abstaining case. 
That work prescribes how to do semi-supervised binary prediction without abstentions, in a strongly optimal sense, 
and further discusses their setting which we summarize here. 


In that no-abstention setting, 
we are given an ensemble of predictors $\cH = \{ h_1, \dots, h_p \}$ 
and i.i.d. unlabeled (``test") data $x_1, \dots, x_n$ on which predictions are evaluated, 
as well as i.i.d. labeled data from the same distribution. 
The ensemble's predictions on the unlabeled data are denoted by $\vF$:
\begin{align*}
\vF = 
 \begin{pmatrix}
   h_1(x_1) & h_1(x_2) & \cdots & h_1 (x_n) \\
   \vdots   & \vdots    & \ddots &  \vdots  \\
   h_p(x_1)  &  h_p (x_2)  & \cdots &  h_p (x_n)
 \end{pmatrix}
\end{align*}
Vector notation is used for the rows and columns of $\vF$: 
$\vh_i = (h_i (x_1), \cdots, h_i (x_n))^\top$ and $\vx_j =
(h_1 (x_j), \cdots, h_p (x_j))^\top$, 
so each data example can be considered as the $p$-vector of the ensemble's predictions on it. 
The test set has some binary labels 
which are allowed to be randomized, 
represented by values in $[-1,1]$ instead of just the two values $\{ -1, 1\}$. 
So it is convenient to write the labels on the test data as $\vz = (z_1; \dots; z_n) \in [-1,1]^n$. 
\footnote{For example, a value $z_i = \frac{1}{2}$ indicates $y_i = +1\;\text{w.p.}\; \frac{3}{4} $.} 

The main idea 
is to formulate the 
prediction problem as a
two-player zero-sum game between a predictor and an adversary.
In this game, the predictor is the first player, 
who plays $\vg = (g_1; g_2; \dots; g_n)$, 
a randomized label prediction $g_j \in [-1,1]$ for each example $\{x_j\}_{j=1}^{n}$. 
The adversary then sets the true labels $\vz \in [-1,1]^n$. 
So the expected classification error suffered by the predictor on example $j \in [n]$ is $\ell (z_j, g_j) := \frac{1}{2} (1 - z_j g_j)$. 

The predictor player does not know $\vz$ when they play, 
but does have some information about it through the ensemble performance on the labeled set. 
When any member of the ensemble $i$ is known to perform to a certain degree on the test data (it has a certain error rate), 
its predictions $\vh_i$ on the test data are a reasonable guide to $\vz$, 
and correspondingly give us information by constraining $\vz$. 
Each hypothesis in the ensemble contributes a constraint to an intersection of constraint sets which contains $\vz$. 

Accordingly, assume the predictor player has knowledge of a \emph{correlation vector}
$\vb \in (0, 1]^p$ such that $
\displaystyle \forall i \in [p] , \quad \frac{1}{n} \sum_{j=1}^n h_i (x_j) z_j \geq b_i 
$, 
i.e. $ \frac{1}{n} \vF \vz \geq \vb$. 
When the ensemble is comprised of binary classifiers, the $p$ inequalities represent upper bounds on individual classifier test error rates, 
which can be estimated w.h.p. from the labeled set 
with uniform convergence bounds as used by the standard supervised ERM procedure (\cite{BF15}).
\footnote{Two-sided bounds can also be dealt with readily, without changing the form of these or our results (see Sec. \ref{sec:discext}). 
We proceed with the one-sided inequality formulation in stating our theoretical results, following BF.}
So in this game-theoretic formulation, the adversary plays under ensemble error constraints defined by $\vb$.



The predictor's goal is to 
\emph{minimize the worst-case expected test error} 
(w.r.t. the randomized labeling $\vz$), written as  
$ \ell (\vz, \vg) := \frac{1}{n} \sum_{j=1}^{n} \ell (z_j, g_j) $.
This worst-case goal can be cast as the following optimization, a game:
\begin{align}
\label{eq:game1pred} 
V_{\pred} := \min_{\vg \in [-1,1]^n} \; \max_{\substack{ \vz \in [-1,1]^n , \\ \frac{1}{n} \vF \vz \geq \vb }} \; \ell (\vz, \vg) 
\end{align}

The predictor faces a learning problem here: 
finding an optimal strategy $\vg^*$ realizing the minimum in \eqref{eq:game1pred}. 
This strategy guarantees good worst-case performance on the unlabeled dataset w.r.t. any possible true labeling $\vz$, 
with an upper bound of $\ell (\vz, \vg^*) \leq V_{\pred}$ on the expected test error. 
The main result of the BF work describes this strategy, which realizes the minimax equilibrium of the game \eqref{eq:game1pred}.
\begin{thm}[\cite{BF15}]
\label{thm:gamesolngen}
Using a weight vector $\sigma \geq \vzero^p$ over $\cH$, 
the vector of unlabeled \emph{scores} is
$\vF^\top \sigma = (\vx_1^\top \sigma, \dots, \vx_n^\top \sigma)$, 
whose magnitudes are the \emph{margins}. 
Define the prediction \emph{potential well} for the 0-1 loss:
$ \Psi (m) := \max (\abs{m}, 1) $.
Also define the \emph{prediction slack function} 
$\displaystyle \gamma (\sigma) := - \vb^\top \sigma + \frac{1}{n} \sum_{j=1}^n \Psi ( \vx_{j}^\top \sigma )$. 
The value of the game \eqref{eq:game1pred} is 
$ V_{\pred} 
= \frac{1}{2} \min_{\sigma \geq \vzero^p} \gamma (\sigma) := \frac{1}{2} \gamma (\sigma^*) $. 
The minimax optimal predictions are defined as follows:
for all $j \in [n]$,
\begin{align}
\label{eq:gipredform}
g_j^* := g_j (\sigma^*) := 
\min \lrp{ 1, \max (\vx_{j}^\top \sigma^*, -1) }
\end{align}
\end{thm}

The potential well encodes the contribution of each unlabeled example to the worst-case error. 
Intuitively, low-margin examples are actually favored to ensure generalization to the test data, 
which will not occur if the algorithm commits overaggressively to its predictions, as for high-margin unlabeled examples. 


A similar argument to that of Thm. \ref{thm:gamesolngen} gives a ``weak duality" observation. 
\begin{observation}[\cite{BF15}]
\label{obs:slacksubopt}
For any $\sigma_0 \geq \vzero^p$, 
the worst-case loss after playing $\vg (\sigma_0)$ is bounded by 
$$ \max_{\substack{ \vz \in [-1,1]^n , \\ \frac{1}{n} \vF \vz \geq \vb }} \; \ell (\vz, \vg (\sigma_0) )
\leq \frac{1}{2} \gamma (\sigma_0) $$
\end{observation}

All this prescribes how to aggregate the given ensemble predictions $\vF$ on the test set: 
\textbf{learn} by minimizing the slack function $\gamma (\sigma)$, finding the weights $\sigma^*$ that achieve $V$, 
and \textbf{predict} with $g_j (\sigma^*)$ on any test example, as indicated in \eqref{eq:gipredform}.


In practice, $\gamma (\sigma)$ is minimized approximately in our statistical learning setting. 
If $\sigma \approx \sigma^*$, then the prediction vector $\vg (\sigma)$ suffers worst-case error $\approx V$ by Obs. \ref{obs:slacksubopt}. 
Since $\gamma$ is 1-Lipschitz and convex and depends on $\vF$ only through the average potential of unlabeled examples, 
such an approximate solution $\approx \sigma^*$ can be found efficiently 
by stochastic optimization methods like minibatch SGD (as in \cite{BF15, BF15b}).



\section{Abstaining with a Fixed Cost}
\label{sec:abstcost}

For the rest of this paper, we build on the minimax setup of Section \ref{sec:setup}, 
extending it to learning an abstaining classifier. 
We first specify how to model the decision-making process of the abstaining classifier we learn. 
The classifier is allowed to abstain on each test example $j \in [n]$ 
with some probability $1 - p_j$, predicting $g_j$ with the remaining probability $p_j \in [0,1]$. 
Therefore, the classifier now plays two vectors: the probabilities of prediction $\vp \in [0,1]^n$, and the predictions $\vg \in [-1,1]^n$.
The adversary again plays the true labels $\vz$, constrained as in Sec. \ref{sec:setup}.

The objective function must reflect the tradeoff between abstaining and erring. 
In the literature, this is often modeled by assessing a pre-specified cost $c$ for an abstention, 
somewhere between the loss of predicting correctly and guessing randomly. 
This induces the algorithm to be self-aware about abstaining -- it abstains whenever it would expect to incur a higher loss, $> c$, 
by predicting against the adversary, 
so that $c$ can be taken to be the largest conditional error probability that is considered tolerable (\cite{C70}). 

Formally, given a value of $c \geq 0$, the game proceeds as follows: 
\begin{enumerate}
\item
Algorithm plays probabilities of prediction $\vp \in [0,1]^n$, and predictions $\vg \in [-1,1]^n$. 
\item
Adversary plays true labels $\vz \in [-1,1]^n$ such that  $\frac{1}{n} \vF \vz \geq \vb$. 
\item
Alg. suffers expected abstaining loss $\ell_{c} (z_j, g_j) := p_j \ell (z_j, g_j) + (1 - p_j) c$ on each test example. 
\end{enumerate}

The goal is now to minimize the worst-case expected \emph{abstaining} loss $\ell_{c}$ on the test data: 
\begin{align}
\label{eq:abstlnrinzcost}
V_{c} &:= \min_{\vg \in [-1,1]^n} \; \min_{\substack{ \vp \in [0,1]^n }} \; \max_{\substack{ \vz \in [-1,1]^n , \\ \frac{1}{n} \vF \vz \geq \vb }} \; 
\frac{1}{n} \sum_{j=1}^{n} \ell_{c} (z_j, g_j)
\end{align}

Our goal is now to find the optimal $\vg_c^* , \vp_c^*$ which minimize \eqref{eq:abstlnrinzcost}. 
These can again be described by defining a potential function. 

\begin{defn}
Define the \textbf{abstaining potential} well given abstaining cost $c$:
\begin{align}
\label{eq:defofabstpot}
\Psi (m, c) = 
\begin{cases}
\abs{m} + 2c (1 - \abs{m}) \quad & \qquad \abs{m} \leq 1 \\
\abs{m} \quad & \qquad \abs{m} > 1
\end{cases}
\end{align}
\end{defn}
Note that this generalizes the potential well of Sec. \ref{sec:setup}, 
which we continue to write as a univariate function $\Psi$, so that $\Psi (m) = \Psi (m, \frac{1}{2})$. 
Abstaining potential wells are plotted as a function of $m$ for different values of $c$ in the left of Fig. \ref{fig:abstpot}. 
As $c$ decreases with all else held equal, the potential's shape changes so that more examples tend to have lower-magnitude margins 
at the optimum, and so they get abstained upon more.

\begin{figure}
 \begin{minipage}[t]{.48\linewidth}
 \vspace{0pt}
\includegraphics[width=\textwidth]{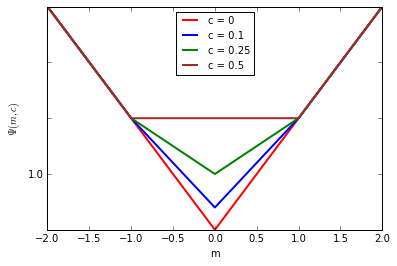}
 \end{minipage}
  \begin{minipage}[t]{.48\linewidth}
     \vspace{0pt}
     \centering
     \includegraphics[width=\textwidth]{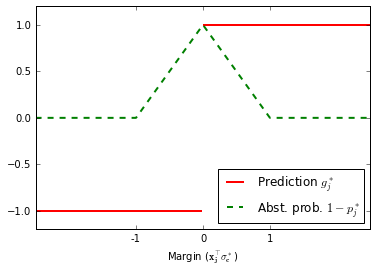}
    \end{minipage}
\caption{\small 
Left: $\Psi (m, c)$ for various $c$. 
Right: optimal predictions $g_j^*$ and abstain probabilities $1 - p_j^*$, 
as a function of score $\vx^\top \sigma_c^*$. 
}
\label{fig:abstpot}
\end{figure}

\begin{thm}
\label{thm:abstcostsoln}
Define the \emph{abstaining slack function} 
$\displaystyle \gamma (\sigma, c) := - \vb^\top \sigma + \frac{1}{n} \sum_{j=1}^{n} \Psi \lrp{ \vx_{j}^\top \sigma, c}$. 
The minimax value of the game \eqref{eq:abstlnrinzcost} 
for any positive cost $c \leq \frac{1}{2}$ is\footnote{The proof (in Appendix \ref{sec:apdxproofs}) shows that if $c > \frac{1}{2}$, the learning problem reduces to the no-abstaining case.}
$\displaystyle
V_{c} = \frac{1}{2} \min_{\sigma \geq \vzero^p} \gamma (\sigma, c)
$. 
If $\sigma_{c}^* \geq \vzero^p$ is the minimizing weight vector in this optimization, 
the minimax optimal predictions $\vg_c^*$ and prediction probabilities $\vp_c^*$ can be defined for each unlabeled example $j \in [n]$ as: 
\begin{align*}
p_{c,j}^* = \min \lrp{ 1, \abs{\vx_j^\top \sigma_{c}^*} }
\qquad , \qquad
g_{c,j}^* = \sgn \lrp{ \vx_j^\top \sigma_{c}^* }
\end{align*}
\end{thm}

Analogous to Section \ref{sec:setup}, we have reduced the learning problem to a convex optimization of a well-behaved function, 
and the optimal prediction on each test example depends only on that example's learned score.
Just as before, this means that the cost-sensitive semi-supervised abstain scenario can be optimally solved efficiently. 

It is easy to see that $\Psi (m, c)$ is increasing in $c$ for any $m$, so for any $c \leq \frac{1}{2}$, 
\begin{align*}
V_{c} 
&= \frac{1}{2} \min_{\sigma \geq \vzero^p} \gamma (\sigma, c) 
\leq \frac{1}{2} \min_{\sigma \geq \vzero^p} \gamma (\sigma) = V_{\pred}
\end{align*}
which shows that allowing abstaining always helps in the worst case; 
its benefit can be quantified in a way that depends intimately on the distribution of ensemble predictions on the (unlabeled) data.

Similarly to Sec. \ref{sec:setup}, the key quantity which encodes the interdependencies of the ensemble on unlabeled data is the weight vector $\sigma_c^*$, 
which is in general different from the no-abstaining solution $\sigma_{c = 1/2}^* = \sigma^*$. 
The optimal strategies for the abstaining classifier are simple, and are independent of $c$ given the $\sigma_c^*$-weighted margin -- 
in particular, after abstaining it is no longer optimal to ``hedge," or randomize, the predictions at all. 

Instead, the prediction $\vg_c^*$ is just a $\sigma_c^*$-weighted majority vote, 
while the optimal classifier abstains with nonzero probability on exactly the examples that it would hedge in the non-abstain case (those with margin $\leq 1$). 
Observe that Thm. \ref{thm:abstcostsoln} for $c = \frac{1}{2}$ reduces to the non-abstaining algorithm of Section \ref{sec:setup} 
despite the apparent difference in $\vg^*$, 
since abstaining in this case is equivalent to predicting one of the two labels uniformly at random. 
At the other extreme of $c = 0$, it can easily be verified that the potential well is minimized by abstaining on all examples, 
so that they are all at the bottom of the well with zero margin.

To summarize, we have established a clean minimax solution to the problem of abstaining with a specified cost. 
To our knowledge, previous work in the fixed-cost model considered a supervised setting with no unlabeled data, 
so we believe this is the first principled semi-supervised abstaining classifier in this model.


\section{Predicting with Constrained Abstain Rate}
\label{sec:abstconstr}

However, the fixed-cost model can still be an unsatisfactory way to study the central abstain-error tradeoff, 
because there may be no clear way to choose $c$.
In this section, we directly study the central tradeoff between abstain rate and error rate without positing an explicitly known cost $c$, 
and derive Fig. \ref{fig:absterrtradeoff}. 

As shown in Fig. \ref{fig:absterrtradeoff}, classifiers on the Pareto frontier minimize their error rate under a given constraint on abstain rate. 
So the algorithm is now given a constant $\alpha \in [0,1]$; 
we simply restrict our abstain rate to be $\leq \alpha$, and derive the minimax optimal error-minimizing strategy 
among all predictors with abstain rate thus restricted. 
The protocol of the game against the adversary is now a slight alteration of Sec. \ref{sec:abstcost}'s game: 
\begin{enumerate}
\item
Algorithm plays probabilities of prediction $\vp \in [0,1]^n$ such that $\frac{1}{n} \vone^\top \vp \geq 1 - \alpha$, and predictions $\vg \in [-1,1]^n$. 
\item
Adversary plays true labels $\vz \in [-1,1]^n$ such that  $\frac{1}{n} \vF \vz \geq \vb$. 
\item
Algorithm suffers expected error $p_j \ell (z_j, g_j)$ on each test example. 
\end{enumerate}

This deals directly with the abstain-error tradeoff described in the introduction. 
Raising $\alpha$ always lowers the classifier's error rate in this scenario, 
because the examples on which it abstains by default contribute zero to the expected error, so abstaining can only lower the error contribution relative to predicting. 

As in the previous section, our goal is to 
minimize the worst-case expected error when predicting on the test data 
(w.r.t. the randomized labeling $\vz$), which we write 
$ \ell_{\vp} (\vz, \vg) := \frac{1}{n} \sum_{j=1}^{n} p_j \ell (z_j, g_j) := \frac{1}{2n} \sum_{j=1}^{n} p_j (1 - z_j g_j) $. 
We can again write our worst-case prediction scenario as a zero-sum game:
\begin{align}
\label{eq:costeqsolnzo} 
&V (\alpha) 
:= \min_{\vg \in [-1,1]^n} \min_{\substack{ \vp \in [0,1]^n , \\ \frac{1}{n} \vone^\top \vp \geq 1 - \alpha }} \; \max_{\substack{ \vz \in [-1,1]^n , \\ \frac{1}{n} \vF \vz \geq \vb }} \; \ell_{\vp} (\vz, \vg) \nonumber \\
&= \min_{\substack{ \vg \in [-1,1]^n , \\ \vp \in [0,1]^n , \\ \frac{1}{n} \vone^\top \vp \geq 1 - \alpha }} \; \max_{\substack{ \vz \in [-1,1]^n , \\ \frac{1}{n} \vF \vz \geq \vb }} \; \frac{1}{2n} \sum_{j=1}^{n} p_j (1 - z_j g_j) 
\end{align}

Next, we exactly compute the optimal $\vg^*$ (and $V (\alpha)$) given $\alpha$, 
and derive an efficient algorithm for learning them. 


\subsection{Solving the Game}

The solution to the game \eqref{eq:costeqsolnzo} depends on the optimum of an appropriate abstaining potential as before.

\begin{thm}
\label{thm:gamesolnabstzo}
Given some $\alpha \in (0,1)$, the minimax value of the game \eqref{eq:costeqsolnzo} is $V (\alpha) = $
\begin{align}
\label{eq:constrabstsolnzo}
\frac{1}{2} \max_{ \lambda \geq 0 } \lrb{ \min_{\sigma \geq \vzero^p} \gamma \lrp{ \sigma, \frac{\lambda}{2}} - \lambda \alpha }
\end{align} 
If $\sigma^*_{\alpha} \geq \vzero^p$ is the minimizing weight vector in this optimization, 
the minimax optimal predictions $\vg^*$ and prediction probabilities $\vp^*$ can be defined for each example $j \in [n]$ in the test set. 
\begin{align*}
p_{j}^* = \min \lrp{ 1, \abs{\vx_j^\top \sigma^*_{\alpha}} }
\qquad , \qquad
g_{j}^* = \sgn \lrp{ \vx_j^\top \sigma^*_{\alpha} }
\end{align*}
\end{thm}

The learning problem of Thm. \ref{thm:gamesolnabstzo} closely resembles that of Section \ref{sec:abstcost}, 
with $\lambda$ replacing $2c$ in the same solution path as $c$ and $\alpha$ are varied. 
This close correspondence is perhaps not surprising when viewed as a manifestation of Lagrange duality; the cost $c$ is dual to the abstain rate $\alpha$. 
However, it appears to be a new contribution to the abstaining literature. 

For the rest of this subsection, we discuss the structure of the solution in Thm. \ref{thm:gamesolnabstzo}. 
Define the maximizing value of $\lambda$ in \eqref{eq:constrabstsolnzo}, written as a function of the abstain rate, to be $\lambda^* (\alpha)$. 
Also write the minimizing $\sigma$ as a function of $\lambda$ as $\sigma^* (\lambda) = \argmin_{\sigma \geq \vzero^p} \gamma (\sigma, \lambda)$. 
Now 
\begin{align*}
\pderiv{\gamma (\sigma, \lambda) }{\lambda}
&= \frac{1}{n} \sum_{j=1}^{n} \pderiv{\Psi \lrp{ \vx_{j}^\top \sigma, \frac{\lambda}{2}} }{\lambda} 
= \frac{1}{n} \sum_{j=1}^{n} \lrb{1 - \abs{\vx_{j}^\top \sigma}}_{+}
\end{align*}
(where $\lrb{m}_{+} := \max (0, m)$), 
so $\gamma (\sigma, \lambda)$ is linear in $\lambda$.  
Therefore, the maximand over $\lambda$ in \eqref{eq:constrabstsolnzo} is concave, because it is a minimum of linear functions: 

\begin{prop}
\label{prop:propsofw}
$w (\lambda) := \displaystyle \min_{\sigma \geq \vzero^p} \gamma (\sigma, \lambda) $ is concave.
\end{prop}

This concavity means $w' (\lambda)$ is decreasing, so we can deal with its pseudoinverse $(w')^{-1} (\alpha) := \inf\{ \lambda \geq 0 : w' (\lambda) \leq \alpha \}$.

Thm. \ref{thm:gamesolnabstzo} presents the value $V (\alpha) $ as the solution to a saddle point problem, 
which may at first appear opaque. 
However, $\lambda$ is the Lagrange parameter for the upper-bound constraint $\alpha$, so that any $\alpha$ maps to a $\lambda^* (\alpha)$, and vice versa. 
\begin{prop}
\label{prop:bijlbdaalpha}
For a given $\alpha$, the $\lambda$ value corresponding to $\alpha$ is $\lambda^*(\alpha) = (w')^{-1} (\alpha - 1)$, 
such that $\frac{1}{n} \sum_{j=1}^{n} p_j^* (\sigma^* [\lambda^*(\alpha)] , \lambda^*(\alpha)) = 1 - \alpha$. 
\end{prop}

Prop. \ref{prop:bijlbdaalpha} implies that the maximizing value $\lambda^* (\alpha)$ is decreasing in $\alpha$, 
again identifying the bijection between $\lambda \geq 0$ and $\alpha \in [0,1]$. 
The bijection is such that the abstain rate is exactly $\alpha$, 
so that as we might expect, the adversary should never leave slack in the abstain rate constraint. 

Returning to the saddle point problem in Thm. \ref{thm:gamesolnabstzo}, 
it is not clear how to efficiently calculate $\lambda^* (\alpha)$ for a given $\alpha$. 
However, one can try the algorithm with $\lambda$ as an input parameter, governing the optimal weight vector $\sigma^* (\lambda)$. 
By Thm. \ref{thm:gamesolnabstzo}, this amounts to minimizing $\gamma(\sigma, c)$, 
which results in an abstain rate that decreases monotonically with increasing $\lambda$ (Prop. \ref{prop:bijlbdaalpha}). 
This all means that a desired abstain rate of $\alpha$ can be achieved by learning classifiers with a few different abstaining costs, 
essentially line-searching for the appropriate cost of abstaining (see Sec. \ref{sec:experiments}). 

\subsection{The Pareto Frontier}

As discussed earlier, the central tradeoff here is between $\alpha$ and the minimax attainable error rate $V (\alpha)$. 
Using our tools, we can describe this Pareto frontier in our general setting, and fully explain Fig. \ref{fig:absterrtradeoff}.

\begin{thm}
\label{thm:maintradeoff}
For $\alpha \geq 0$, 
$V (\alpha)$ is the convex, decreasing, nonnegative function
\begin{align}
V(\alpha) = \frac{1}{2} \lrp{ w (\lambda^* (\alpha)) - \alpha \lambda^* (\alpha) }
\end{align}
In particular, $V (0) = V_{\pred}$ and $V(1) = 0$, and $V' (\alpha) = - \frac{1}{2} \lambda^* (\alpha)$. 
\end{thm}

For intuition into this result, 
consider that since $V (\alpha)$ is convex and decreasing on $[0,1]$, 
$V' (\alpha) \leq \frac{V(1) - V(\alpha)}{1 - \alpha} = - \frac{V(\alpha)}{1 - \alpha}$. 
So the curve $V(\alpha)$ is always decreasing at least as fast as its secant to the point $(\alpha=1, V(1) = 0)$. 
At any $\alpha$, this means there is a marginal benefit to abstaining, relative to the error incurred by predicting with $\alpha$ constant. 
Thm. \ref{thm:maintradeoff} may be of independent interest as a complete minimax characterization of a bicriterion optimization problem; 
many learning problems can be written as multicriterion optimizations, but are typically scalarized (\cite{CBL06, JS08}).

\section{Related Work and Discussion}
\label{sec:discrelwork}

\subsection{Related Work}
\label{sec:relwork}

There is a growing theoretically principled literature on expressing unsureness in binary classification in a fully supervised setting. 
This ranges from selective sampling in an online learning setting (e.g. \cite{CBL06, LLW08}) to more benign statistical learning settings like ours. 
We focus on the latter here as being more relevant: although our formulation is game-theoretic and transductive, 
the first-moment constraints imposed by the ensemble in this paper (i.e. $\vb$) handle data that is i.i.d. or similar, 
and cannot be readily imposed in online learning scenarios. 
Though there are many practically useful approaches to abstaining (e.g. \cite{HLS97, LTPD06}), 
they are most often Bayesian and are not analyzed under model misspecification. 

Characterizing the abstain/error tradeoff has been a central goal of previous work since the earliest decision-theoretic work on the topic (\cite{C57, C70}), 
which focuses on scenarios where the conditional label probability (or a plug-in estimator) is known. 
Such work (\cite{C70}) establishes the Bayes optimal classifier under these assumptions and links it to the cost $c$ of abstaining. 
Thereafter, statistical learning work has focused on analyzing classifiers in a fully supervised fixed-cost setting with respect to the Bayes optimal classifier. 
Starting with the influential paper of \cite{BW08}, 
such work includes consistency guarantees (\cite{GRKC09, YW10}) 
for regularized surrogate loss minimization approaches for learning linear combinations of basis functions;
adaptations of the SVM / hinge loss minimization are by far the most well-studied algorithms that arise from these papers (\cite{WY11}). 

A separate, purely theoretical line of work (\cite{EYW10, EYW11, ZC14}) analyzes the problem for more general binary ensembles, 
building from classical learning theory when there is a classifier in the ensemble with low (or no) error. 
The algorithms in such work basically track disagreements among a \emph{version space} of ``good" classifiers, 
and are shown to eventually converge on the best classifier in the space as ERM would. 
A similar setting is used by \cite{FMS04} to devise an abstaining averaging-based scheme for a general ensemble;
they too adopt an ERM-style analysis of convergence to the best single classifier's predictions, though with possibly faster convergence than ERM. 

Notably, the work (\cite{EYW10}) gives bounds on the entire abstain-predict tradeoff curve for a general ensemble (albeit in the realizable case). 
It also raises the tantalizing possibility of learning to predict with zero error for some $\alpha < 1$, 
when the learned classifier is able to abstain wherever it is unsure and predict perfectly otherwise. 
This is conceivable in our framework because $g_j^*$ does not randomize its predictions. 
Since $V(1) = 0$ and $V (\alpha)$ is convex and decreasing, zero-error prediction can happen exactly when $V' (\alpha) = - \frac{1}{2} \lambda^* (\alpha) = 0$ for some $\alpha < 1$. 
Working through the definitions, this happens when $w' (0) = 0$, 
a criterion for when zero-error prediction is possible based on the unlabeled data distribution.


\subsection{Novelty and Extensions}
\label{sec:discext}

Our approach to abstaining is novel in several ways compared to prior work. 
Using unlabeled data allows us to optimize the abstaining loss $\ell_{c}$ directly, 
an option not available to previous fully supervised work -- 
to our knowledge, there are no previous theoretically motivated \emph{semi-supervised} approaches to abstaining. 

The finite-sample minimax guarantees in this paper hold with respect to all possible abstaining classifiers 
learned with the given ensemble, not just linear combinations. 
In our case, being given a finite ensemble is \emph{not a restrictive assumption}, 
because the origin and structure of $\vF$ is completely unrestricted in our proofs. 
Each predictor in the ensemble can be thought of as a feature; 
it has so far been convenient to think of it as binary, following the perspective of binary classifier aggregation, but it could as well be real-valued. 
An unlabeled example $\vx$ is simply a vector of this ensemble of features, a familiar notion. 
It is fortunate in practice that the minimax aggregation of the ensemble in our setting 
happens to rely on an efficiently computed linear combination. 
But this is not assumed a priori, unlike in aforementioned prior work. 

When the ensemble consists of binary classifiers, 
we incorporate voting behavior to cancel errors in the ensemble 
and guarantee performance at least as good as the version-space methods' (i.e. ERM), and better 
in any situation in which the given ensemble and $\vb$ permit such a guarantee. 
When considering an ensemble of general features, since we do not relax $\ell_{c}$, 
our guarantees turn out to imply asymptotic Bayes consistency results (when such results are possible given $\vF$ and $\vb$). 

A recent sequel of BF (\cite{BF15b}) shows that an easily computed reweighting of the rows of $\vF$ 
suffices to handle ensembles of \emph{specialists}, 
each of which predicts only on some arbitrary subset of the data and abstains on the rest. 
(The resulting reweighted data can be real-valued even when $\vF$ is binary, consistent with the above view of ensembles as features.)
This also applies to our approach (see Appendix \ref{sec:specialistsetup} for details), 
and this flexibility can be useful in practice, as seen in Sec. \ref{sec:experiments}. 

Our techniques and the resulting algorithms generalize in other ways, 
including to other loss functions which may be asymmetric or even nonconvex (Appendix \ref{sec:general}). 
Though Sec. \ref{sec:abstcost} takes the cost of abstaining to be a fixed scalar, 
its main proof (of Thm. \ref{thm:abstcostsoln}) goes through essentially unchanged 
with a different cost $c_i$ for each example $i \in [n]$ (i.e. if errors are more tolerable a priori on some unlabeled data than others). 

Though we build on the BF semi-supervised framework of \cite{BF15, BF15b, BF16}, 
the above versatility and advantages of our approach do not follow trivially from their prior work, 
and owe much to our analysis techniques. 
This paper's basic setting is significantly more general than previous work, 
considering minimax behavior that trades off between two types of actions rather than a standard decision theory setting. 
Our abstaining algorithms' seeming simplicity is deceptive, in light of the data-dependent nuances involved in this behavior.

\paragraph{$L_{\infty}$ Bounds on $\vb$ $\iff$ $L_1$ Regularization of $\sigma$.}
A useful extension in practice is to change the way in which the ensemble's label correlations $\vb$ provide us information. 
For instance, think of the ensemble as consisting of binary classifiers, 
and write the natural plugin estimate of the label correlations from the labeled data as $\hat{\vb}$. 
Instead of applying a correction to $\hat{\vb}$ to make it a lower bound on $\frac{1}{n} \vF \vz$ as discussed so far, 
uniform convergence theory (\cite{V82}) suggests that $\vnorm{\frac{1}{n} \vF \vz - \hat{\vb}}_{\infty}$ is low w.h.p., 
inspiring a result that uses this information to define the minimax optimum:


\begin{thm}
\label{thm:abstL1reg}
Define the game value with $L_{\infty}$ bounds and abstain cost $c$:
\begin{align*}
V^{\infty}_{c} &:= \min_{\vg \in [-1,1]^n} \; \min_{\substack{ \vp \in [0,1]^n }} \; \max_{\substack{ \vz \in [-1,1]^n , \\ \vnorm{\frac{1}{n} \vF \vz - \vb}_{\infty} \leq \epsilon }} \; 
\frac{1}{n} \sum_{j=1}^{n} \ell_{c} (z_j, g_j)
\end{align*}
Then 
$\displaystyle V^{\infty}_{c}
= \frac{1}{2} \min_{\sigma \in \RR^p} \left[ \gamma ( \sigma , c ) + \epsilon \vnorm{\sigma}_1 \right]
$.
\end{thm}
This reformulates the learning problem as unconstrained $L_1$-regularized optimization, 
rather than unregularized optimization constrained to the positive orthant. 
The regularized formulation applies to all the results in this paper and the appendices, 
as noted in the proof of Thm. \ref{thm:abstL1reg}. 
Conveniently, one setting of $\epsilon$ suffices for all $c$ (and thus the Pareto frontier), 
since it is clear from the $L_\infty$ primal problem that $\epsilon$ does not depend on $c$.


\section{Experiments}
\label{sec:experiments}

This section augments our theory with some proof-of-concept experiments using the algorithms derived in this paper. 

\paragraph{Baselines}
A large majority of empirical work on abstaining algorithms in the research literature has involved 
adapting a fully supervised linear classifier like the SVM to abstain, typically by rejecting all data with low enough margin ($\leq \delta$ for some $\delta \geq 0$). 
\footnote{We were made aware of a very recent exception in theory (\cite{CDM16}), 
but their experiments still fall into this category.}
We therefore use two prevailing such baselines for abstaining classifiers. 
The first is a plain SVM minimizing the hinge loss, with 
margin threshold $\delta$ chosen to minimize abstaining loss $\ell_{c}$.
The second is a double hinge (DH) loss minimizer introduced in recent years (\cite{BW08}, which prescribes $\delta \approx 0.5$) 
and investigated in several sequel papers (see Appendix \ref{sec:expdetails}). 

These serve as appropriate baselines to our classifier, a linear separator with weights $\sigma$ and a ``soft" abstain region 
(cf. discussion of Sec. \ref{sec:discext}). 
Like SVM-type linear abstaining classifiers, a practical advantage of our algorithms is their interpretability in terms of familiar notions like score/margin, 
on which the minimax optimality results shed light. 
However, an important distinction is that our method encourages \emph{low} margins (because $\Psi$ is convex), 
hedging against uncertainty to achieve minimax optimal behavior and contrasting with the max-margin intuition of the SVM. 

\paragraph{Method}
Our framework is well suited to i.i.d. data despite being derived in a transductive setting. 
The learning problem, a convex optimization of the abstaining slack function $\gamma (\sigma, \lambda)$, 
can be solved with highly scalable stochastic optimization algorithms (cf. end of Sec. \ref{sec:setup}); we use conjugate gradient. 
Given $\vb$, there are conveniently no parameters to tune; 
however, in practice the remaining challenge is to estimate $\vb$ from labeled data. 
We find it best to use the $L_1$-regularized formulation of Thm. \ref{thm:abstL1reg} with the plugin estimator $\hat{\vb}$ from the labeled set, 
and cross-validate to set the regularization parameter $\epsilon$.

In keeping with the original intuition of the formulation, we experiment on datasets with binary features 
from the UCI and LIBSVM repositories. 
These arise in many different contexts, including one-hot encoding of categorical features and of text. 
In addition to running our basic method on them to find $\sigma_c^*$ (``Basic" in the table), 
we also run the method when considering each feature to be a specialist which abstains on examples where it is zero, 
contributing information only when nonzero (``Spec"). 
Appendix \ref{sec:expdetails} has more information on the protocol.


\begin{table}[h]
\caption{Abstaining Loss ($\ell_c$) Results} \label{table:results}
\begin{center}
\begin{tabular}{ | c | c | c | c | c | c | c | }
  \hline
   &
  \multicolumn{2}{|c|}{\texttt{adult}} &
  \multicolumn{2}{|c|}{\texttt{secstr}} &
  \multicolumn{2}{|c|}{\texttt{mushroom}} \\
  \hline \hline  
$c$ & 0.2  & 0.4  & 0.2 & 0.4 & 0.2 & 0.4 \\ \hline \hline
SVM & 0.09  & 0.15  & 0.17 & 0.29 & 0.01 & 0.01 \\ \hline
DH & 0.10  & 0.15  & 0.17 & 0.28 & 0.003 & 0.02 \\ \hline
Basic & 0.11  & 0.13  & 0.20 & 0.30 & 0.01 & 0.01 \\ \hline
Spec & 0.10  & 0.12  & 0.18 & \textbf{0.24} & 0.00 & 0.01 \\ \hline
\end{tabular}
\end{center}
\end{table}


Some of our results
are in Table \ref{table:results}, 
with significant results (10 Monte Carlo trials, paired $t$-tests) in bold when applicable. 

They indicate that our methods are broadly competitive with the baselines, 
despite our completely contrasting min-margin intuition and optimization. 
We believe this justifies our theory's recommendations to abstain on low-margin data.
Using features as specialists appears to slightly help performance, as might be expected; 
we are currently evaluating the method on much larger text datasets in follow-up empirical work. 

The form of $\vp^*$ suggests margin-based applications like active sampling (\cite{TK02}).
The framework also is especially suited to combining abstaining classifiers in useful ways like cascades (\cite{MWLN08, TS13}). 
We believe that such empirical explorations, and a more thorough performance evaluation of this paper's core algorithms 
and its other untried ones, 
would be future work of interest.




\newpage


\bibliography{abstain_refs}{}
\bibliographystyle{alpha}

\newpage
\appendix


\section{Experiment Details}
\label{sec:expdetails}

\subsection{Specialists in the Ensemble}
\label{sec:specialistsetup}

Here we describe how to incorporate specialists in our formulation (i.e. members of the ensemble that can abstain) 
in a very general manner, following \cite{BF15b}. 
Suppose that for a classifier $i \in [p]$ and an example $x_j$ ($j \in [n]$), 
the classifier decides to abstain with probability $1 - v_i (x)$. 
and otherwise predict $h_i (x)$. 
Our only assumption on $\{v_i (x_1), \dots, v_i (x_n) \}$ is the reasonable one that  
$\sum_{j=1}^n v_i (x_j) > 0$, so classifier $i$ is not a useless specialist that abstains everywhere.

The information about $\vz$ contributed by classifier $i$ is now not its overall correlation with $\vz$ on the entire test set, 
but rather the correlation with $\vz$ restricted to the test examples on which it does not abstain. 
In other words, for some $[b_{S}]_{i} \in [0,1] $, 
\begin{align}
\label{eq:specconstr}
\sum_{j=1}^n \lrp{ \frac{v_i (x_j)}{ \sum_{k=1}^n v_i (x_k) } } h_i (x_j) z_j \geq [b_{S}]_{i}
\end{align}
For convenience write $\rho_i (x_j) := \frac{v_i (x_j)}{ \sum_{k=1}^n v_i (x_k) }$ (a distribution over $j \in [n]$). 

Now we can replace the label correlation vector $\vb$ by $\vb_S$, 
and the unlabeled data matrix $\vF$ by the following matrix $\vS$, formed by reweighting each row (``feature"): 
\begin{align}
\vS = n 
 \begin{pmatrix}
   \rho_1 (x_1) h_1(x_1) & \rho_1 (x_2) h_1(x_2) & \cdots & \rho_1 (x_n) h_1 (x_n) \\
   \rho_2 (x_1) h_2(x_1) & \rho_2 (x_2) h_2(x_2) & \cdots & \rho_2 (x_n) h_2 (x_n) \\
   \vdots   & \vdots    & \ddots &  \vdots  \\
   \rho_p (x_1) h_p(x_1) & \rho_p (x_2) h_p (x_2)  & \cdots & \rho_p (x_n) h_p (x_n)
 \end{pmatrix}
\end{align}
When ensemble constituent $i$ is not a specialist, $\rho_i (x_j) = \frac{1}{n} \forall i$; 
when the entire ensemble consists of non-specialists, $\vS = \vF$.

Therefore, the constraint \eqref{eq:specconstr} can be written as $ \frac{1}{n} \vS \vz \geq \vb_S$. 
The $L_{\infty}$ norm constraints can be rewritten in exactly the same fashion.

\subsection{Experimental Protocol}

The datasets were obtained from the LIBSVM (\cite{libsvm11}) and UCI repositories. 
As they are binary, no preprocessing was used, 
and a linear kernel (used for the results) worked at least as well as polynomial or Gaussian kernels for the baselines. 

For our semi-supervised methods, we estimate $\vb$ using the labeled set. 
As a fair comparison, the same amount of labeled data is used for this estimation and for training the baseline classifiers 
(1000 for mushroom, 10000 for the other two datasets). 
We hold out a modicum of data for validation for each dataset
(1000 for all three). 
leaving some remaining data, which is used by our method as unlabeled data. 
Setting higher values of $\epsilon$ tends to increase the abstain rate by lowering $\vnorm{\sigma}_1$ and thereby lowering margins.

Our baseline DH loss minimizer follows previous work which established it (\cite{BW08, YW10}).
We also tried minimizing the DH loss with $L_1$ and $L_2$ regularization 
(studied respectively in \cite{WY11} and \cite{GRKC09}), cross-validating to tune the regularization parameter. 
These consistently performed about as well as our DH loss baseline, or worse; their results are omitted.

\section{Proofs and Ancillary Results}
\label{sec:apdxproofs}

\subsection{Proofs of Results in Main Paper Body}

\begin{proof}[Proof of Thm. \ref{thm:abstcostsoln}]
Examining \eqref{eq:abstlnrinzcost}, note that the objective function is linear in $\vz$, so the constrained maximization over $\vz$ is basically  
\begin{align}
\max_{\substack{ \vz \in [-1,1]^n , \\ \frac{1}{n} \vF \vz \geq \vb }} \;\; - \frac{1}{n} &\sum_{i=1}^{n} p_j z_j g_j 
\label{eq:abstinnerdualcostzo}
= \max_{\substack{ \vz \in [-1,1]^n , \\ \frac{1}{n} \vF \vz \geq \vb }} \;\; - \frac{1}{n} \vz^\top [\vp \circ \vg ]
= \min_{\sigma \geq \vzero^p} \left[ - \vb^\top \sigma + \frac{1}{n} \vnorm{\vF^\top \sigma - (\vp \circ \vg) }_1 \right]
\end{align}
where the last equality uses Lemma \ref{lem:gamegeng}, a basic application of Lagrange duality. 

Substituting \eqref{eq:abstinnerdualcostzo} into \eqref{eq:abstlnrinzcost} and simplifying, 
\begin{align}
V_{c} 
&= \frac{1}{2} \min_{\vg \in [-1,1]^n} \min_{\substack{ \vp \in [0,1]^n }} \; \lrb{ \frac{1}{n} \sum_{j=1}^{n} \lrp{ p_j + 2 (1-p_j) c } + 
\max_{\substack{ \vz \in [-1,1]^n , \\ \frac{1}{n} \vF \vz \geq \vb }} \;\; - \frac{1}{n} \sum_{j=1}^{n} p_j z_j g_j } \nonumber \\
&= c + \frac{1}{2} \min_{\vg \in [-1,1]^n} \min_{\substack{ \vp \in [0,1]^n }} \; \lrb{ \frac{1}{n} \sum_{j=1}^{n} p_j (1- 2c) + 
\min_{\sigma \geq \vzero^p} \left[ - \vb^\top \sigma + \frac{1}{n} \vnorm{\vF^\top \sigma - (\vp \circ \vg) }_1 \right] } \nonumber \\
\label{eq:interabstoptcostzo}
&= \frac{1}{2} \min_{\sigma \geq \vzero^p} \lrb{ - \vb^\top \sigma + \min_{\substack{ \vp \in [0,1]^n }} \frac{1}{n} \sum_{j=1}^{n} \lrb{ 2c + p_j (1- 2c) + 
\min_{g_j \in [-1,1]} \left| \vx_j^\top \sigma - p_j g_j \right| } }
\end{align}

Now we consider only the innermost minimization over $g_j$ in \eqref{eq:interabstoptcostzo} for any $j$ (following the procedure of the main proofs of \cite{BF15, BF16}).  

If $p_j = 0$, then the inner minimization's objective is simply $\abs{\vx_j^\top \sigma}$. 
Otherwise, it can be seen that 
\begin{align}
\label{eq:costacvxmmand}
\min_{g_j \in [-1,1]} \left| \vx_j^\top \sigma - p_j g_j \right|
= p_j \min_{g_j \in [-1,1]} \left| \frac{\vx_j^\top \sigma}{p_j} - g_j \right|
= p_j \lrp{ \Psi \lrp{\frac{\vx_j^\top \sigma}{p_j}} - 1 }
= \lrb{ \abs{ \vx_{j}^\top \sigma } - p_j }_{+}
\end{align}

with the minimizer $g_j^* = \clip_{-1}^{1} \lrp{\frac{\vx_j^\top \sigma}{p_j}}$. 
So we can rewrite \eqref{eq:interabstoptcostzo} as 
\begin{align}
\label{eq:penultcostabst}
V_{c} 
&= \frac{1}{2} \min_{\sigma \geq \vzero^p} \lrb{ - \vb^\top \sigma + \frac{1}{n} \sum_{j=1}^{n} \min_{p_j \in [0,1] } \lrb{ 2c + p_j (1- 2c)  + 
\lrb{ \abs{ \vx_{j}^\top \sigma } - p_j }_{+} } } 
\end{align}

Now the minimand over $p_j$ is convex in $p_j$; since $c \in (0, \frac{1}{2}]$, 
it is decreasing in $p_j$ for $p_j < \abs{ \vx_{j}^\top \sigma }$ and increasing for $p_j > \abs{ \vx_{j}^\top \sigma }$. 
Therefore, it is minimized when $p_j^* = \min (\abs{ \vx_{j}^\top \sigma }, 1)$. 
Substituting this $p_j^*$ into the minimization over $p_j$ in \eqref{eq:penultcostabst} 
gives the form of $\Psi ( \cdot, c)$, and therefore the result. 

Note that if $c > \frac{1}{2}$, the minimand over $p_j$ in \eqref{eq:penultcostabst} 
is monotonically decreasing in $p_j$, and so is optimized with $p_j^* = 1$ for all $j$. 
So in this case the minimax solution is to never abstain, and the objective function 
reduces to the prediction slack function $\gamma (\sigma)$ regardless of $c$. 
\end{proof}

\begin{proof}[Proof of Theorem \ref{thm:gamesolnabstzo}]
We begin exactly as in the proof of Theorem \ref{thm:abstcostsoln}, by rewriting the constrained maximization over $\vz$ in the definition of $V_{\alpha}$: 
\begin{align}
V_{\alpha} 
&= \frac{1}{2} \min_{\vg \in [-1,1]^n} \min_{\substack{ \vp \in [0,1]^n , \\ \frac{1}{n} \vone^\top \vp \geq 1 - \alpha }} \; \lrb{ \frac{1}{n} \sum_{j=1}^{n} p_j + 
\max_{\substack{ \vz \in [-1,1]^n , \\ \frac{1}{n} \vF \vz \geq \vb }} \;\; - \frac{1}{n} \sum_{j=1}^{n} p_j z_j g_j } \nonumber \\
&= \frac{1}{2} \min_{\vg \in [-1,1]^n} \min_{\substack{ \vp \in [0,1]^n , \\ \frac{1}{n} \vone^\top \vp \geq 1 - \alpha }} \; \lrb{ \frac{1}{n} \sum_{j=1}^{n} p_j + 
\min_{\sigma \geq 0^p} \left[ - \vb^\top \sigma + \frac{1}{n} \vnorm{\vF^\top \sigma - (\vp \circ \vg) }_1 \right] } \nonumber \\
\label{eq:interabstoptsolnzo}
&= \frac{1}{2} \min_{\sigma \geq 0^p} \lrb{ - \vb^\top \sigma + \min_{\substack{ \vp \in [0,1]^n , \\ \frac{1}{n} \vone^\top \vp \geq 1 - \alpha }} \frac{1}{n} \sum_{j=1}^{n} \lrb{ p_j + 
\min_{g_j \in [-1,1]} \left| \vx_j^\top \sigma - p_j g_j \right| } } \nonumber \\
&= \frac{1}{2} \min_{\sigma \geq 0^p} \lrb{ - \vb^\top \sigma + \min_{\substack{ \vp \in [0,1]^n , \\ \frac{1}{n} \vone^\top \vp \geq 1 - \alpha }} \frac{1}{n} \sum_{j=1}^{n} \lrb{ p_j + 
\lrb{ \abs{ \vx_{j}^\top \sigma } - p_j }_{+} } }
\end{align}
where at each step we follow the sequence of steps in the proof of Theorem \ref{thm:abstcostsoln}. 
Again, the minimizer $g_j^* = \clip_{-1}^{1} \lrp{\frac{\vx_j^\top \sigma}{p_j}}$. 

Now we simplify \eqref{eq:interabstoptsolnzo}, introducing a Lagrange parameter for the constraint on $\vp$ involving $\alpha$. 
\begin{align}
V_{\alpha} 
&= 
\frac{1}{2} \min_{\substack{ \vp \in [0,1]^n , \\ \frac{1}{n} \vone^\top \vp \geq 1 - \alpha }} \min_{\sigma \geq 0^p} 
\lrb{ - \vb^\top \sigma + \frac{1}{n} \sum_{j=1}^{n} \lrb{ p_j + \lrb{ \abs{ \vx_{j}^\top \sigma } - p_j }_{+} } } \nonumber \\
\label{eq:abstouterzo}
&= 
\frac{1}{2} \min_{ \vp \in [0,1]^n } \max_{ \lambda \geq 0 } \min_{\sigma \geq 0^p} 
\lrb{ - \vb^\top \sigma + \lambda (1 - \alpha) + \frac{1}{n} \sum_{j=1}^{n} \lrb{ p_j (1-\lambda) + \lrb{ \abs{ \vx_{j}^\top \sigma } - p_j }_{+} } } 
\end{align}

For any $j \in [n]$, the summand is convex in $p_j$. 
So we can apply the minimax theorem (\cite{SF12}, p. 144) to \eqref{eq:abstouterzo}, 
moving $\min_{ \vp \in [0,1]^n }$ inside the maximization while maintaining equality.
This yields 
\begin{align*}
V_{\alpha} 
&=
\frac{1}{2} \max_{ \lambda \geq 0 } \min_{\sigma \geq 0^p} 
\lrb{ - \vb^\top \sigma + \lambda (1 - \alpha) + \frac{1}{n} \sum_{j=1}^{n} \min_{ p_j \in [0,1] } \lrb{ p_j (1-\lambda) + \lrb{ \abs{ \vx_{j}^\top \sigma } - p_j }_{+} } } 
\end{align*}
using the fact that the $L^{\infty}$ constraints over $\vp$ decouple over the coordinates. 

For any $\lambda \geq 0$, the inner minimand over $p_j$ is decreasing for $p_j \leq \abs{\vx_{j}^\top \sigma}$ and increasing for $p_j > \vx_{j}^\top \sigma$. 
So the minimizing $p_j^* = \min(\abs{\vx_{j}^\top \sigma}, 1)$. 
Substituting this in the minimization gives 
\begin{align*}
V_{\alpha} &=
\frac{1}{2} \max_{ \lambda \geq 0 } \min_{\sigma \geq 0^p} 
\lrb{ - \vb^\top \sigma + \lambda (1 - \alpha) + \frac{1}{n} \sum_{j=1}^{n} \lrp{ \Psi \lrp{ \vx_{j}^\top \sigma , \frac{\lambda}{2} } - \lambda } } 
\end{align*}
giving the result.
\end{proof}

%

\begin{proof}[Proof of Theorem \ref{thm:maintradeoff}]
The definition of $V(\alpha)$ follows by substituting $\lambda^* (\alpha)$ in the definition of $V$ in terms of $\lambda$.
Nonnegativity is immediate by the primal definition of $V$. 
As for convexity, 
\begin{align*}
V' (\alpha)
&= \frac{1}{2} \lrp{ \lrb{ \pderiv{\lambda^* (\alpha)}{\alpha}} w' (\lambda^* (\alpha)) - \alpha \pderiv{\lambda^* (\alpha)}{\alpha} - \lambda^* (\alpha)} \\
&= \frac{1}{2} \lrp{ \lrb{ \pderiv{\lambda^* (\alpha)}{\alpha}} \lrb{ w' (\lambda^* (\alpha)) - \alpha} - \lambda^* (\alpha)} \\
&= - \frac{1}{2} \lambda^* (\alpha) \leq 0
\end{align*}
where the last equality is by subgradient conditions defining $\lambda^* (\alpha)$ as the maximizer over $\lambda$.

We know $\lambda^* (\alpha)$ is decreasing with $\alpha$, so $V'' (\alpha) \geq 0$ and $V$ is convex. 
The boundary conditions at $\alpha = 0,1$ follow by computation. 
\end{proof}

\begin{proof}[Proof of Theorem \ref{thm:abstL1reg}]
The proof is exactly like that of Thm. \ref{thm:abstcostsoln}, 
except that Lemma \ref{lem:linfinnerdual} is used instead of Lemma \ref{lem:gamegeng} -- 
all other steps remain unchanged. 

\textbf{Note} that this replacement is general. 
All the results of this paper that solve games (Thms. \ref{thm:abstcostsoln}, \ref{thm:gamesolnabstzo}, \ref{thm:abstcostsolngenloss}) 
use Lemma \ref{lem:gamegeng} to handle the primal constraints $\frac{1}{n} \vF \vz \geq \vb$, 
and replacing these with $L_{\infty}$ norm primal constraints is in all cases equivalent to using Lemma \ref{lem:linfinnerdual} 
instead of Lemma \ref{lem:gamegeng}. 
\end{proof}


\subsection{Ancillary Results}

\begin{lem}[\cite{BF15}]
\label{lem:gamegeng}
For any $\va \in \RR^n$,
\begin{align*}
\max_{\substack{ \vz \in [-1,1]^n , \\ \frac{1}{n} \vF \vz \geq \vb }} \;\;\frac{1}{n} \vz^\top \va 
\;=\; \min_{\sigma \geq 0^p} \left[ - \vb^\top \sigma + \frac{1}{n} \vnorm{\vF^\top \sigma + \va}_1 \right]
\end{align*}
\end{lem}
\begin{proof}[Proof of Lemma \ref{lem:gamegeng}]
We have
\begin{align}
\label{eq:primallagrange}
\displaystyle \max_{\substack{ \vz \in [-1,1]^n , \\ \vF \vz \geq n \vb }} \;\; \frac{1}{n} \vz^\top \va 
&= \frac{1}{n} \max_{\vz \in [-1,1]^n} \; \min_{\sigma \geq 0^p} \;  \lrb{ \vz^\top \va + \sigma^\top (\vF \vz - n \vb) } \\
\label{eq:duallagrange}
&\stackrel{(a)}{=} \frac{1}{n} \min_{\sigma \geq 0^p} \; \max_{\vz \in [-1,1]^n} \; \lrb{ \vz^\top (\va + \vF^\top \sigma) - n \vb^\top \sigma } \\
\label{eq:gameinnerdual}
&= \frac{1}{n} \min_{\sigma \geq 0^p} \; \lrb{ \vnorm{\va + \vF^\top \sigma}_1 - n \vb^\top \sigma } 
= \min_{\sigma \geq 0^p} \left[ - \vb^\top \sigma + \frac{1}{n} \vnorm{\vF^\top \sigma + \va}_1 \right]
\end{align}
where $(a)$ is by the minimax theorem (\cite{CBL06}, Thm. 7.1). 
\end{proof}

\begin{proof}[Proof of Prop. \ref{prop:bijlbdaalpha}]
The result follows from checking the first-order optimality condition 
$$ \pderiv{w (\lambda) - \lambda \alpha}{\lambda} \Big\vert_{\lambda = \lambda^*(\alpha)} = 0$$
\end{proof}

\begin{lem}[\cite{BF16}]
\label{lem:linfinnerdual}
For any $\va \in \RR^n$,
\begin{align*}
\max_{\substack{ \vz \in [-1,1]^n , \\ \vnorm{\frac{1}{n} \vF \vz - \vb}_{\infty} \leq \epsilon }} \;\frac{1}{n} \vz^\top \va 
\;=\; \min_{\sigma \in \RR^p} \left[ - \vb^\top \sigma + \frac{1}{n} \vnorm{\vF^\top \sigma + \va}_1 + \epsilon \vnorm{\sigma}_1 \right]
\end{align*}
\end{lem}
\begin{proof}[Proof of Lemma \ref{lem:linfinnerdual}]
\begin{align*}
\max_{\substack{ \vz \in [-1,1]^n , \\ \vnorm{\frac{1}{n} \vF \vz - \vb}_{\infty} \leq \epsilon }} \;\frac{1}{n} \vz^\top \va 
&= \max_{\substack{ \vz \in [-1,1]^n , \\ \frac{1}{n} \vF \vz - \vb \leq \epsilon \vone^n , \\ - \frac{1}{n} \vF \vz + \vb \leq \epsilon \vone^n }} \;\frac{1}{n} \vz^\top \va \\
&= \frac{1}{n} \;\max_{\vz \in [-1,1]^n} \min_{\lambda, \xi \geq 0^p} \lrb{ \vz^\top \va 
+ \lambda^\top (- \vF \vz + n \vb + n \epsilon \vone^n) + \xi^\top (\vF \vz - n \vb + n \epsilon \vone^n) } \\
&= \frac{1}{n} \;\min_{\lambda, \xi \geq 0^p} \max_{\vz \in [-1,1]^n} \lrb{ \vz^\top 
(\va + \vF^\top (\xi - \lambda)) + \lambda^\top ( n \vb + n \epsilon \vone^n) + \xi^\top (- n \vb + n \epsilon \vone^n) } \\
&= \frac{1}{n} \;\min_{\lambda, \xi \geq 0^p} \lrb{  
\vnorm{\va + \vF^\top (\xi - \lambda)}_1 - n \vb^\top (\xi - \lambda) + n \epsilon \vone^{n \top} (\xi + \lambda) }
\end{align*}
where the interchanging of min and max is justified by the minimax theorem (\cite{CBL06}), 
since the objective is linear in each variable and one of the constraint sets is closed. 

Suppose for some $j \in [n]$ that $\xi_j > 0$ and $\lambda_j > 0$. 
Then subtracting $\min(\xi_j, \lambda_j)$ from both does not affect the value $[\xi - \lambda]_{j}$, 
but always decreases $[\xi + \lambda]_{j}$, and therefore always decreases the objective function. 
Therefore, we can w.l.o.g. assume that $\forall j \in [n]: \min(\xi_j, \lambda_j) = 0$. 
Defining $\sigma_j = \xi_j - \lambda_j$ for all $j$ (so that $\xi_j = [\sigma_j]_+$ and $\lambda_j = [\sigma_j]_-$), 
the last equality above becomes
\begin{align*}
\frac{1}{n} \;\min_{\lambda, \xi \geq 0^p} \lrb{  
\vnorm{\va + \vF^\top (\xi - \lambda)}_1 - n \vb^\top (\xi - \lambda) + n \epsilon \vone^{n \top} (\xi + \lambda) }
= \frac{1}{n} \;\min_{\sigma \in \RR^p} \lrb{  
\vnorm{\va + \vF^\top \sigma}_1 - n \vb^\top \sigma + n \epsilon \vnorm{\sigma}_1 }
\end{align*}
\end{proof}


\section{Abstaining with General Losses}
\label{sec:general}

In this section, we show that the entire abstaining formulation set up in this paper also extends to a very general class of loss functions 
introduced in a recent paper (\cite{BF16}), 
which generalizes the semi-supervised BF setup of Sec. \ref{sec:setup} for the 0-1 classification error. 

The following derivations thereby show that such losses admit efficient algorithms to learn abstaining classifiers, 
with similar optimality guarantees to those presented in the main body of this paper. 
Here, we re-analyze the fixed cost model of Section \ref{sec:abstcost}, 
this time allowing the aggregated classifier's prediction performance to be measured by other losses than the 0-1 loss. 
This allows us to minimax optimally handle a variety of binary prediction tasks, 
such as predicting Bayesian-style probabilities (log loss) or coping with asymmetric misclassification costs. 

This section does not contain new algorithmic ideas relative to the rest of the paper, 
only demonstrating that our analysis incorporates abstaining in an extremely flexible and general manner. 
It does not rely crucially on properties of the 0-1 loss (like its linearity over unlabeled data), 
and indeed extends to a vast class of losses. 
All notation defined here holds only for this section, superseding the equivalent notation 
used in the main body of the paper for the 0-1 loss. 

\subsection{Loss Functions}
In this subsection, we briefly note the setup of \cite{BF16}, defining a general class of prediction loss functions; 
see that paper for further discussion. 
On any single test point with randomized binary label $z_j \in [-1, 1]$, 
our expected performance upon predicting $g_j$, with respect to the randomization of $z_j$, 
is now measured by a loss function $\ell (z_j, g_j)$, which generalizes the $0-1$ loss. 
Also, 
$$ \ell (z_j, g_j) = \lrp{\frac{1+z_j}{2}} \ell (1, g_j) + \lrp{\frac{1-z_j}{2}} \ell (-1, g_j) 
:= \lrp{\frac{1+z_j}{2}} \ell_{+} (g_j) + \lrp{\frac{1-z_j}{2}} \ell_{-} (g_j) $$
where we conveniently write $\ell_{+} (g_j) := \ell (1, g_j)$ and $\ell_{-} (g_j) := \ell (-1, g_j)$. 
We call $\ell_{\pm}$ the \emph{partial losses}. 

We make an assumption on $\ell_{+} (\cdot)$ and $\ell_{-} (\cdot)$:

\begin{assumption}
\label{ass:loss}
Over the interval $(-1,1)$, $\ell_{+} (\cdot)$ is decreasing 
and $\ell_{-} (\cdot)$ is increasing, 
and both are twice differentiable.
\end{assumption}
This clearly includes standard convex partial losses.
We view Assumption \ref{ass:loss} as natural, because the loss function intuitively measures discrepancy to the true label $\pm 1$. 

We refer to an ``increasing" continuous function $f$ as being any nondecreasing function, i.e. such that $x > y \implies f(x) \geq f(y)$. 
Any such function has a pseudoinverse $f^{-1} (y) = \min \{ y : f(x) = y \}$.


The predictor's goal is to 
\emph{minimize the worst-case expected loss on the test data} 
(w.r.t. the randomized labeling $\vz$), written as  
$ \ell (\vz, \vg) := \frac{1}{n} \sum_{j=1}^{n} \ell (z_j, g_j) $.
This worst-case goal can be cast as the following minimax game:
\begin{align}
\label{eq:game1predgenloss} 
V_{\pred} &:= \min_{\vg \in [-1,1]^n} \; \max_{\substack{ \vz \in [-1,1]^n , \\ \frac{1}{n} \vF \vz \geq \vb }} \; \ell (\vz, \vg) \\
&= \frac{1}{2} \min_{\vg \in [-1,1]^n} \; \max_{\substack{ \vz \in [-1,1]^n , \\ \frac{1}{n} \vF \vz \geq \vb }} \;\; 
\frac{1}{n} \sum_{j=1}^{n} \lrb{ \ell_{+} (g_j) + \ell_{-} (g_j) + z_j \lrp{ \ell_{+} (g_j) - \ell_{-} (g_j) } }
\end{align}

The learning problem faced by the predictor can be solved, 
finding an optimal strategy $\vg^*$ realizing the minimum in \eqref{eq:game1pred} for general losses. 
This strategy guarantees good worst-case performance on the unlabeled dataset w.r.t. any possible true labeling $\vz$, 
with an upper bound of $\ell (\vz, \vg^*) \leq V$ on the loss. 
$\vg^*$ turns out to depend componentwise on a linear combination of the input hypotheses, through a sigmoid nonlinearity.   

We describe this formally. 
Define the loss-based \textbf{score function} $ \Gamma : [-1,1] \mapsto \RR$ as 
$$ \Gamma (g) := \ell_{-} (g) - \ell_{+} (g) $$ 
(We also write the vector $\Gamma (\vg) $ componentwise with $[\Gamma (\vg)]_j = \Gamma (g_j)$ for convenience, 
so that $\Gamma (\vh_i) \in \RR^n $ and $\Gamma (\vx_j) \in \RR^p $.)
Observe that by our assumptions, $\Gamma (g)$ is increasing on its domain. 

With these in mind, we can set up the solution to the game \eqref{eq:game1predgenloss}.
\begin{defn}
Define the \textbf{potential well}
\begin{align}
\Psi (m) = 
\begin{cases} 
- m + 2 \ell_{-} (-1)  \qquad & \mbox{ \; if \; } m \leq \Gamma (-1) \\ 
\ell_{+} (\Gamma^{-1} (m)) + \ell_{-} (\Gamma^{-1} (m))  \qquad & \mbox{ \; if \; } m \in \lrp{ \Gamma (-1) , \Gamma (1)} \\ 
m + 2 \ell_{+} (1)  & \mbox{ \; if \; } m \geq \Gamma (1)
\end{cases}
\end{align}
\end{defn}

The minimax equilibrium of the game \eqref{eq:game1predgenloss} can now be described.
\begin{thm}[\cite{BF16}]
\label{thm:gamesolngenloss}
With respect to a weight vector $\sigma \geq \vzero^p$ over $\cH$, 
the vector of \textbf{ensemble predictions} is
$\vF^\top \sigma = (\vx_1^\top \sigma, \dots, \vx_n^\top \sigma)$, 
whose elements' magnitudes are the \textbf{margins}. 
The \textbf{prediction slack function} is
\begin{align}
\label{eqn:slack}
\gamma (\sigma, \vb) := \gamma (\sigma) := - \vb^\top \sigma + \frac{1}{n} \sum_{j=1}^n \Psi ( \vx_{j}^\top \sigma )
\end{align}
The minimax value of the game \eqref{eq:game1predgenloss} is 
$$ \min_{\vg \in [-1,1]^n} \; \max_{\substack{ \vz \in [-1,1]^n , \\ \frac{1}{n} \vF \vz \geq \vb }} \; \ell (\vz, \vg) = V_{\pred} 
= \frac{1}{2} \gamma (\sigma^*) = \frac{1}{2} \min_{\sigma \geq \vzero^p} \lrb{ - \vb^\top \sigma + \frac{1}{n} \sum_{j=1}^n \Psi ( \vx_{j}^\top \sigma ) } $$
The minimax optimal predictions are defined as follows:
for all $j \in [n]$,
\begin{align}
\label{eq:gipredformconstr}
g_j^* := g_j (\sigma^*) = 
\begin{cases} 
-1  \qquad & \mbox{ \; if \; } \vx_{j}^\top \sigma^* \leq \Gamma (-1) \\ 
\Gamma^{-1} (\vx_{j}^\top \sigma^*) \qquad & \mbox{ \; if \; } \vx_{j}^\top \sigma^* \in \lrp{ \Gamma (-1) , \Gamma (1)} \\ 
1  & \mbox{ \; if \; } \vx_{j}^\top \sigma^* \geq \Gamma (1)
\end{cases}
\end{align}
\end{thm}

This has a similar structure to the 0-1 case, but using the loss-dependent $\Gamma$ function. 
The function $\Gamma$ plays the same role as a ``link function" used in Generalized Linear Models (GLMs) (\cite{MN89}), 
as well as of an activation function for an artificial neuron in neural network applications.

Just as for the 0-1 loss, learning can be done efficiently, with guarantees that follow directly from optimization guarantees, 
because the potential well $\Psi$ (and therefore slack function) is convex under fairly general conditions. 

\begin{lem}[\cite{BF16}]
\label{lem:cvxpotentialgenloss}
The potential well $\Psi (m)$ is continuous and 1-Lipschitz. 
It is also convex under \emph{any} of the following conditions:
\begin{enumerate}
\item
The partial losses $\ell_{\pm} (\cdot)$ are convex over $(-1,1)$.
\item
The loss function $\ell (\cdot, \cdot)$ is a proper loss (\cite{RW10}). 
\item
$\ell_{-}' (x) \ell_{+}'' (x) \geq \ell_{-}'' (x) \ell_{+}' (x)$ for all $x \in (-1,1)$.
\end{enumerate}
\end{lem}

These conditions encompass all losses commonly used in efficient ERM-based learning. 
We assume that one of these conditions hold and \emph{$\Psi$ is convex} hereafter in this section. 
We also assume $\Psi$ to be differentiable, purely for convenience in stating our results; 
the example of the 0-1 loss shows that a nondifferentiable $\Psi$ is no obstacle to analysis. 

Using these assumptions, we now turn to analyze the fixed-cost model of abstaining, 
for losses that satisfy any of the conditions in Lemma \ref{lem:cvxpotentialgenloss}.


\subsection{Predicting and Abstaining with Specified Cost}

We combine the analysis techniques of Section \ref{sec:abstcost} with the setup of partial losses $\ell_{\pm}$ in the previous subsection, 
seeking to minimize the worst-case expected abstaining loss on the test data:
\begin{align}
\label{eq:origgame}
V_{c} 
&= \min_{\vg \in [-1,1]^n} \; \min_{\substack{ \vp \in [0,1]^n }} \; \max_{\substack{ \vz \in [-1,1]^n , \\ \frac{1}{n} \vF \vz \geq \vb }} \;\; 
\frac{1}{n} \sum_{j=1}^{n} \lrb{ p_j \lrb{ \lrp{\frac{1+z_j}{2}} \ell_{+} (g_j) + \lrp{\frac{1-z_j}{2}} \ell_{-} (g_j) } + (1 - p_j) c } \\
\label{eq:abstlnrinz}
&= \frac{1}{2} \min_{\vg \in [-1,1]^n} \; \min_{\substack{ \vp \in [0,1]^n }} \; \max_{\substack{ \vz \in [-1,1]^n , \\ \frac{1}{n} \vF \vz \geq \vb }} \;\; 
\frac{1}{n} \sum_{j=1}^{n} \lrp{ p_j \lrb{ \ell_{+} (g_j) + \ell_{-} (g_j) + z_j \lrp{ \ell_{+} (g_j) - \ell_{-} (g_j) } } + 2 (1 - p_j) c}
\end{align}

The goal is now to find the optimal $\vg_c^* , \vp_c^*$ which minimize the above. 
In doing this, the following function and its derivatives play a major role.
\begin{lem}
\label{lem:defofallpotsgenloss}
Define the function 
\begin{align*}
G (p_j, \vx_j, \sigma) = p_j \Psi \lrp{\frac{\vx_j^\top \sigma}{p_j}}
\end{align*}
Then $ G (0, \vx_j, \sigma) = \lim_{p_j \to 0^+} p_j \Psi \lrp{\frac{\vx_j^\top \sigma}{p_j}} = \abs{\vx_j^\top \sigma}$. 
Also, $G$ is convex in $p_j$, so that
\begin{align*}
K_{j, \sigma} (p_j) := \pderiv{G (p_j, \vx_j, \sigma)}{p_j} = \Psi \lrp{\frac{\vx_j^\top \sigma}{p_j}} - \frac{\vx_j^\top \sigma}{p_j} \Psi' \lrp{\frac{\vx_j^\top \sigma}{p_j}}
\end{align*}
is increasing in $p_j$ for $p_j \geq 0$. 
\end{lem}

\begin{proof}[Proof of Lemma \ref{lem:defofallpotsgenloss}]
From the definition of $\Psi$, we have 
$\displaystyle \lim_{p \to 0^+} \frac{\Psi \lrp{\frac{\vx_j^\top \sigma}{p}}}{\frac{\vx_j^\top \sigma}{p}} = \sgn \lrp{\vx_j^\top \sigma}$,
so $\displaystyle \lim_{p \to 0^+} \frac{K (p, \vx_j, \sigma)}{\frac{\vx_j^\top \sigma}{p}} = \sgn \lrp{\vx_j^\top \sigma} - \lim_{p \to 0^+} \Psi' \lrp{\frac{\vx_j^\top \sigma}{p}} = 0$, 
and so $K (0, \vx_j, \sigma) = 0$. 
Moreover, $\pderiv{K (p_j, \vx_j, \sigma)}{p_j} = \frac{(\vx_j^\top \sigma)^2}{p_j^3} \Psi'' \lrp{\frac{\vx_j^\top \sigma}{p_j}} \geq 0$ because $\Psi$ is convex, 
which makes $G$ convex in $p_j$. 
\end{proof}

Since this function is increasing in $p_j$, we can deal with its inverse $K_{j, \sigma}^{-1} (\lambda)$, also an increasing function.

A few other definitions are necessary to set up our main result. 
Define the function 
\begin{align*}
\phi (x) = \Psi (x) - x \Psi' (x)
\end{align*}
Now $\phi (x)$ is nonnegative, because $\phi \lrp{\frac{\vx_j^\top \sigma}{p_j} } = K_{j, \sigma} (p_j)$, which a straightforward calculation verifies is always nonnegative for all $j$. 
Also, $\phi' (x) = - x \Psi'' (x)$, so $\phi (x)$ is decreasing for $x \geq 0$ and increasing for $x < 0$. 

Therefore, we can define $\phi_{+}^{-1}$ to be the inverse of the mapping $\{ (x, \phi(x)) : x \geq 0 \}$, 
and $\phi_{-}^{-1}$ to be the inverse of the mapping $\{ (x, \phi(x)) : x < 0 \}$, 
and $\phi_{+}^{-1}$ is decreasing and $\phi_{-}^{-1}$ is increasing. 
So if $\phi (x) = \lambda$, then $x = \begin{cases} \phi_{+}^{-1} (\lambda) & , x \geq 0 \\ \phi_{-}^{-1} (\lambda) & , x < 0 \end{cases} $.

\begin{defn}
Define the \textbf{abstaining potential well}:
\begin{align}
\label{eq:defofqj}
Q_j (\sigma, \lambda) = 
\begin{cases}
(\vx_j^\top \sigma) \Psi' \lrp{ \phi_{\sgn(\vx_j^\top \sigma)}^{-1} (\lambda) }  \quad & \qquad \lambda \leq K_{j, \sigma} (1) \\
\Psi \lrp{\vx_j^\top \sigma} - \lambda \quad & \qquad \lambda > K_{j, \sigma} (1)
\end{cases}
\end{align}
\end{defn}
$Q_j$ is clearly convex in $\vx_j^\top \sigma$, and therefore convex in $\sigma$. 
Finally, we can compute that for any positive $\lambda_0$, 
\begin{align}
\label{eq:derivofq}
\pderiv{Q_j (\sigma, \lambda)}{\lambda} \Big\vert_{\lambda = \lambda_0} = 
- p_j^* (\sigma, \lambda_0) = 
\begin{cases}
- K_{j, \sigma}^{-1} (\lambda_0) \quad & \qquad \lambda_0 \leq K_{j, \sigma} (1) \\
-1 \quad & \qquad \lambda_0 > K_{j, \sigma} (1)
\end{cases}
\end{align}
so since $K_{j, \sigma}^{-1}$ is increasing, $Q_j$ is concave in $\lambda$.

\begin{thm}
\label{thm:abstcostsolngenloss}
The minimax value of the game \eqref{eq:origgame} is 
$$ V_{c} = \frac{1}{2} \min_{\sigma \geq \vzero^p} \lrb{ - \vb^\top \sigma + \frac{1}{n} \sum_{j=1}^{n} Q_j (\sigma, 2 c) } + c $$
If we define $\sigma_{c}^*$ to be the minimizing weight vector in this optimization, 
the minimax optimal predictions $\vg_c^*$ and prediction probabilities $\vp_c^*$ are defined as follows for each example $j \in [n]$ in the test set. 
\begin{align}
p_{c,j}^* = \min \lrp{ 1, K_{j, \sigma_{c}^*}^{-1} (2c) }
\end{align}
\begin{align}
\label{eq:gipredformconstr}
g_{c,j}^* = 
\begin{cases} 
-1  \qquad & \mbox{ \; if \; } \vx_{j}^\top \sigma_{c}^* \leq p_{c,j}^* \Gamma (-1) \\ 
\Gamma^{-1} \lrp{\frac{\vx_j^\top \sigma_{c}^*}{p_{c,j}^*}} \qquad & \mbox{ \; if \; } \vx_{j}^\top \sigma_{c}^* \in p_{c,j}^* \lrp{ \Gamma (-1) , \Gamma (1)} \\ 
1  & \mbox{ \; if \; } \vx_{j}^\top \sigma_{c}^* \geq p_{c,j}^* \Gamma (1)
\end{cases}
\end{align}
\end{thm}

\begin{proof}[Proof of Theorem \ref{thm:abstcostsolngenloss}]
This proof generalizes the main proof of \cite{BF16} using ideas from the proof of Thm. \ref{thm:abstcostsoln}. 
First, note that $\ell (\vz, \vg)$ is linear in $\vz$, so the constrained maximization over $\vz$ is basically  
\begin{align}
\label{eq:abstinnerdual}
\max_{\substack{ \vz \in [-1,1]^n , \\ \frac{1}{n} \vF \vz \geq \vb }} \;\; \frac{1}{n} \sum_{i=1}^{n} (1 - a_j) z_j \lrp{ \ell_{+} (g_j) - \ell_{-} (g_j) } 
&= \max_{\substack{ \vz \in [-1,1]^n , \\ \frac{1}{n} \vF \vz \geq \vb }} \;\; - \frac{1}{n} \vz^\top [\vp \circ \Gamma (\vg) ] \nonumber \\
&= \min_{\sigma \geq 0^p} \left[ - \vb^\top \sigma + \frac{1}{n} \vnorm{\vF^\top \sigma - (\vp \circ \Gamma (\vg)) }_1 \right]
\end{align}
where the last equality uses Lemma \ref{lem:gamegeng}. 

Substituting \eqref{eq:abstinnerdual} into \eqref{eq:abstlnrinz} and simplifying, 
\begin{align}
V 
&= \frac{1}{2} \min_{\vg \in [-1,1]^n} \min_{\substack{ \vp \in [0,1]^n }} \; \lrb{ \frac{1}{n} \sum_{j=1}^{n} \lrp{ p_j \lrb{ \ell_{+} (g_j) + \ell_{-} (g_j) } + 2 (1-p_j) c } + 
\max_{\substack{ \vz \in [-1,1]^n , \\ \frac{1}{n} \vF \vz \geq \vb }} \;\;  \frac{1}{n} \sum_{j=1}^{n} p_j z_j \lrp{ \ell_{+} (g_j) - \ell_{-} (g_j) } } \nonumber \\
&= \frac{1}{2} \min_{\vg \in [-1,1]^n} \min_{\substack{ \vp \in [0,1]^n }} \; \lrb{ \frac{1}{n} \sum_{j=1}^{n} \lrp{ p_j \lrb{ \ell_{+} (g_j) + \ell_{-} (g_j) } + 2 (1-p_j) c } + 
\min_{\sigma \geq 0^p} \left[ - \vb^\top \sigma + \frac{1}{n} \vnorm{\vF^\top \sigma - (\vp \circ \Gamma (\vg)) }_1 \right] } \nonumber \\
&= \frac{1}{2} \min_{\sigma \geq 0^p} \lrb{ - \vb^\top \sigma + \min_{\substack{ \vp \in [0,1]^n }} \min_{\vg \in [-1,1]^n}
\lrb{ \frac{1}{n} \sum_{j=1}^{n} \lrp{ p_j \lrb{ \ell_{+} (g_j) + \ell_{-} (g_j) } + 2 (1-p_j) c } + \frac{1}{n} \vnorm{\vF^\top \sigma - (\vp \circ \Gamma (\vg))}_1 } } \nonumber \\
\label{eq:costinterabstopt}
&= \frac{1}{2} \min_{\sigma \geq 0^p} \lrb{ - \vb^\top \sigma + \min_{\substack{ \vp \in [0,1]^n }} \frac{1}{n} \sum_{j=1}^{n} \lrb{ 2 (1-p_j) c + 
\min_{g_j \in [-1,1]} \lrb{ p_j (\ell_{+} (g_j) + \ell_{-} (g_j) ) + \left| \vx_j^\top \sigma - p_j \Gamma (g_j) \right| }} }
\end{align}

Now we consider only the innermost minimization over $g_j$ for any $j$. 

If $p_j = 0$ in \eqref{eq:costinterabstopt}, then the inner minimization's objective is simply $\abs{\vx_j^\top \sigma}$. 
Otherwise, we note that the absolute value breaks down into two cases, so the inner minimization's objective can be simplified:
\begin{align}
\label{eq:costacvxmmand}
p_j (\ell_{+} (g_j) + \ell_{-} (g_j)) + \left| \vx_j^\top \sigma - p_j \Gamma (g_j) \right| 
= \begin{cases} 
p_j \lrp{ 2 \ell_{+} (g_j) + \frac{\vx_j^\top \sigma}{p_j}}  \qquad & \mbox{ \; if \; } \frac{\vx_j^\top \sigma}{p_j} \geq \Gamma (g_j) \\ 
p_j \lrp{ 2 \ell_{-} (g_j) - \frac{\vx_j^\top \sigma}{p_j}}  & \mbox{ \; if \; } \frac{\vx_j^\top \sigma}{p_j} < \Gamma (g_j)
\end{cases}
\end{align}

This is nearly the same situation faced in the proof of Theorem \ref{thm:gamesolngen} in \cite{BF16}, except with $\vx_j^\top \sigma$ replaced by $\frac{\vx_j^\top \sigma}{p_j}$. 
So we proceed with that argument, 
finding that: 
\begin{align*}
\min_{g_j \in [-1,1]} \lrb{ p_j (\ell_{+} (g_j) + \ell_{-} (g_j)) + \left| \vx_j^\top \sigma - p_j \Gamma (g_j) \right| }
= p_j \Psi \lrp{\frac{\vx_j^\top \sigma}{p_j}}
\end{align*}
with the minimizing $g_j^*$ being of the same form as Theorem \ref{thm:gamesolngen}, 
but dependent on the quantity $\frac{\vx_j^\top \sigma}{p_j}$ instead of $\vx_j^\top \sigma$. 

So we can rewrite \eqref{eq:costinterabstopt} as 
\begin{align}
V = \mbox{Eq. }\eqref{eq:costinterabstopt}
&= \frac{1}{2} \min_{\sigma \geq 0^p} \lrb{ - \vb^\top \sigma + \frac{1}{n} \sum_{j=1}^{n} \min_{p_j \in [0,1] } \lrb{ 2 (1-p_j) c + 
\min_{g_j \in [-1,1]} \lrb{ p_j (\ell_{+} (g_j) + \ell_{-} (g_j) ) + \left| \vx_j^\top \sigma - p_j \Gamma (g_j) \right| }} } \nonumber \\
&= 
\frac{1}{2} \min_{\sigma \geq 0^p} 
\lrb{ - \vb^\top \sigma + \frac{1}{n} \sum_{j=1}^{n} \min_{p_j \in [0,1] } \lrb{ 2 (1-p_j) c + 
p_j \Psi \lrp{\frac{\vx_j^\top \sigma}{p_j}} } } \nonumber
\end{align}

Now the minimand over $p_j$ is convex in $p_j$, by Lemma \ref{lem:defofallpotsgenloss}.
Using first-order optimality conditions, observe that $p_j^*$ is such that $2 c = K_{j, \sigma} (p_j^*)$, 
if the implied $p_j^*$ is $\leq 1$, and $1$ otherwise. Therefore, 
\begin{align*}
p_j^* = \min \lrp{ K_{j, \sigma}^{-1} (2 c), 1}
\end{align*}
and so we have
\begin{align}
V = \mbox{Eq. }\eqref{eq:costinterabstopt}
&= \frac{1}{2} \min_{\sigma \geq 0^p} \lrb{ - \vb^\top \sigma + \frac{1}{n} \sum_{j=1}^{n} Q_j (\sigma, 2 c) } + c
\end{align}
\end{proof}

\end{document}